\def\year{2021}\relax
\documentclass[letterpaper]{article} 
\usepackage{aaai21}  
\usepackage{times}  
\usepackage{helvet} 
\usepackage{courier}  
\usepackage[hyphens]{url}  
\usepackage{graphicx} 
\usepackage{natbib}  
\usepackage{caption} 
\usepackage{cite}
\usepackage{amsmath,amssymb,amsfonts}
\usepackage{algorithmic}
\usepackage{graphicx}
\usepackage{textcomp}
\usepackage{xcolor}
\def\BibTeX{{\rm B\kern-.05em{\sc i\kern-.025em b}\kern-.08em
		T\kern-.1667em\lower.7ex\hbox{E}\kern-.125emX}}

\usepackage{url}
\usepackage{subcaption}
\usepackage{amsmath}
\usepackage{amsthm}
\usepackage{booktabs}
\usepackage{algorithm}
\usepackage{enumitem}
\usepackage{mathtools}
\usepackage{array}
\newcolumntype{H}{>{\setbox0=\hbox\bgroup}c<{\egroup}@{}}

\newcommand{\bigCI}{\mathrel{\text{\scalebox{1.07}{$\perp\mkern-10mu\perp$}}}}

\newtheorem{theorem}{Theorem}

\usepackage[symbol]{footmisc}

\title{Treatment Effect Estimation with Disentangled Latent Factors}
\author {
        Weijia Zhang\footnote{Corresponding Author},
        Lin Liu,
        Jiuyong Li \\
}
\affiliations {
     University of South Australia
     \\
    weijia.zhang.xh@gmail.com, \{lin.liu, jiuyong.li\}@unisa.edu.au
}

\begin{document}

\maketitle

\begin{abstract}
Much research has been devoted to the problem of estimating treatment effects from observational data; however, most methods assume that the observed variables only contain confounders, i.e., variables that affect both the treatment and the outcome. Unfortunately, this assumption is frequently violated in real-world applications, since some variables only affect the treatment but not the outcome, and vice versa. Moreover, in many cases only the proxy variables of the underlying confounding factors can be observed. 
In this work, we first show the importance of differentiating confounding factors from instrumental and risk factors for both average and conditional average treatment effect estimation, and then we propose a variational inference approach to simultaneously infer latent factors from the observed variables, disentangle the factors into three disjoint sets corresponding to the instrumental, confounding, and risk factors, and use the disentangled factors for treatment effect estimation.
Experimental results demonstrate the effectiveness of the proposed method on a wide range of synthetic, benchmark, and real-world datasets.
\end{abstract}

\section*{Introduction}
Estimating the effect of a treatment on an outcome is a fundamental problem faced by many researchers and has a wide range of applications across diverse disciplines.
In social economy, policy makers need to determine whether a job training program will improve the employment perspective of the workers \cite{Athey2015}. 
In online advertisement, companies need to predict whether an advertisement campaign could persuade a potential buyer into buying the product \cite{Rzepakowski2011}. 


To estimate treatment effect from observational data, the treatment assignment mechanism needs to be independent of the possible outcomes when conditioned on the observed variables, i.e., the unconfoundedness assumption \cite{Rosenbaum1983} needs to be satisfied. With this assumption, treatment effects can be estimated from observational data by adjusting on the confounding variables which affects both the treatment assignment and the outcome. The treatment effect estimation may be biased if not all the confounders are considered in the estimation \cite{Pearl2009}.

From a theoretical perspective, practitioners are tempted to include as many variables as possible to ensure the satisfaction of the unconfoundedness assumption.
This is because confounders can be difficult to measure in the real-world and practitioners need to include noisy proxy variables to ensure unconfoundedness. 
For example, the socio-economic status of patients confounds treatment and prognosis, but cannot be included in the electronic medical records due to privacy concerns. It is often the case that such unmeasured confounders can be inferred from noisy proxy variables which are easier to measure. For instance, the zip codes and job types of patients can be used as proxies to infer their socio-economic statuses \cite{Sauer2013}. 

From a practical perspective, the inflated number of variables included for confounding adjustment reduces the efficiency of treatment effect estimation. 
Moreover, it has been previously shown that including unnecessary covariates is suboptimal when the treatment effect is estimated non-parametrically \cite{Hahn1998,Abadie2006,Haeggstroem2017}. In a high dimensional scenario, eventually many included variables will not be confounders and should be excluded from the set of adjustment variables. 


Most existing treatment estimation algorithms treat the given variables ``as is", and leave the task of choosing confounding variables to the user. 
It is clear that the users are left with a dilemma: on the one hand including more variables than necessary produces inefficient and inaccurate estimators; on the other hand restricting the number of adjustment variables may exclude confounders themselves or proxy variables of the confounders and thus increases the bias of the estimated treatment effects. 
With only a handful of variables, the problem can be avoided by consulting domain experts. 
However, a data-driven approach is required in the big data era to deal with the dilemma.

In this work, we propose a data-driven approach for simultaneously inferencing latent factors from proxy variables and disentangling the latent factors into three disjoint sets as illustrated in Figure \ref{illustration}: the instrumental factors $\mathbf{z}_t$ which only affect the treatment but not the outcome, the risk factors $\mathbf{z}_y$ which only affect the outcome but not the treatment, and the confounding factors $\mathbf{z}_c$ that affect both the treatment and the outcome.
Since our method builds upon the recent advancement of the research on variational autoencoder \cite{Kingma2014}, we name our method Treatment Effect by Disentangled Variational AutoEncoder (TEDVAE).
Our main contributions are:
\begin{itemize}
	\item We address an important problem in treatment effect estimation from observational data, where the observed variable may contain confounders, proxies of confounders and non-confounding variables. 
	\item We propose a data-driven algorithm, TEDVAE, to simultaneously infer latent factors from proxy variables and disentangle confounding factors from the others for a more efficient and accurate treatment effect estimation. 
	\item We validate the effectiveness of the proposed TEDVAE algorithm on a wide range of synthetic datasets, treatment effect estimation benchmarks and real-world datasets. 
\end{itemize}

The rest of this paper is organized as follows. In Section 2, we discuss related works. The details of TEDVAE is presented in Section 3. In Section 4, we discuss the evaluation metrics, datasets and experiment results. Finally, we conclude the paper in Section 5.

\begin{figure}[!t]
	\centering
	\includegraphics[width = 0.9\columnwidth]{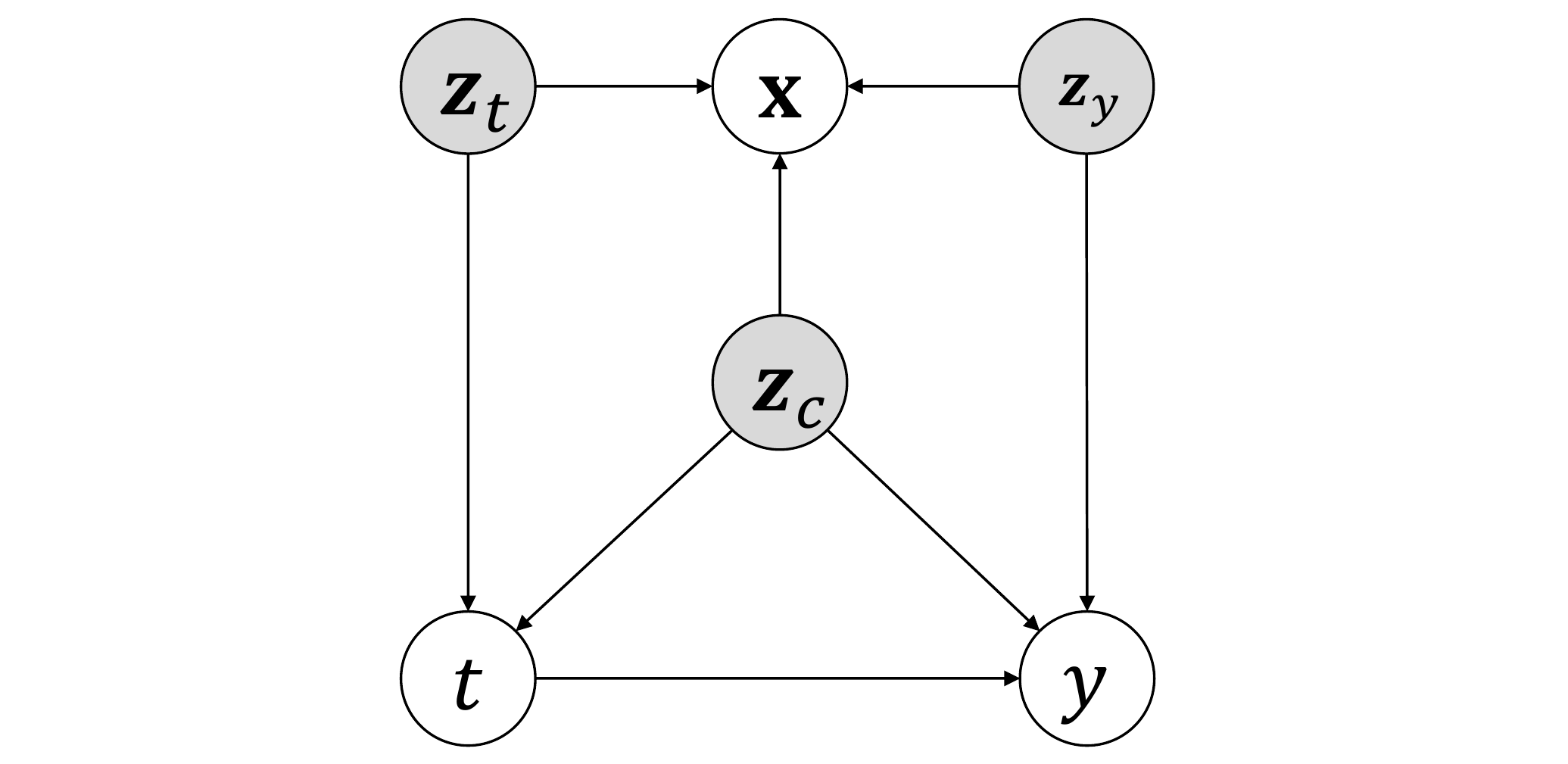}
	\caption{Model diagram for the proposed Treatment Effect with Disentangled Autoencoder (TEDVAE). $t$ is the treatment, $y$ is the outcome. $\mathbf{x}$ is the ``as-is" observed variables which may contain non-confounders and noisy proxy variables. 
		$\mathbf{z}_t$ are factors that affect only the treatment, $\mathbf{z}_y$ are factors that affect only the outcome, and $\mathbf{z}_c$ are confounding factors that affect both treatment and outcome. 
	}
	\label{illustration}
\end{figure}
\section*{Related Work}
Treatment effect estimation has steadily drawn the attentions of researchers from the statistics and machine learning communities. During the past decade,  
several tree based methods \cite {Su2009,Athey2015,Zhang2017,Zhang2018} have been proposed to address the problem by designing a treatment effect specific splitting criterion for recursive partitioning. 
Ensemble algorithms and meta algorithms \cite{Kuenzel2019,Wager2018} have also been explored. For example, Causal Forest\cite{Wager2018} builds ensembles using the Causal Tree \cite{Athey2015} as base learners. X-Learner \cite{Kuenzel2019} is a meta algorithm that can utilize off-the-shelf machine learning algorithms for treatment effect estimation. 

Deep learning based heterogeneous treatment effect estimation methods have attracted increasingly research interest in recent years \cite{Shalit2016,Alaa2018,Louizos2017,Hassanpour2018,Yao2018_Twin,Yoon2018}. 
Counterfactual Regression Net \cite{Shalit2016} and several other methods \cite{Yao2018_Twin,Hassanpour2018} have been proposed to reduce the discrepancy between the treated and untreated groups of samples by learning a representation such that the two groups are as close to each other as possible.
However, their designs do not separate the covariates that only contribute to the treatment assignment from those only contribute to the outcomes. Furthermore, these methods are not able to infer latent covariates from proxies. 

Variable decomposition \cite{Kun2017,Haeggstroem2017} has been previously investigated for average treatment effect estimation. 
Our method has several major differences from the above methods: (i) our method is capable of estimating the individual level heterogeneous treatment effects, where existing ones only focus on the population level average treatment effect; (ii) we are able to identify the non-linear relationships between the latent factors and their proxies, whereas their approach only models linear relationships. Recently, a deep representation learning based method, DR-CFR \cite{Hassanpour2020} is proposed for treatment effect estimation.

Another work closely related to ours is the Causal Effect Variational Autoencoder (CEVAE) \cite{Louizos2017}, which also utilizes variational autoencoder to learn confounders from observed variables. However, CEVAE does not consider the existence of non-confounders, and is not able to learn the separated sets of instrumental and risk factors. 
As demonstrated by the experiments, disentangling the factors significantly improves the performance.

\section*{Method}

\subsection*{Preliminaries}
Let $t_i \in \{0,1\}$ denote a binary treatment where $t_i = 0$ indicates the $i$-th individual receives no treatment (control) and $t_i = 1$ indicates the individual receives the treatment (treated). We use $y_i(1)$ to denote the potential outcome of $i$ if it were treated, and $y_i(0)$ to denote the potential outcome if it were not treated. Noting that only one of the potential outcomes can be realized, and the observed outcome is $y_i = (1-t_i)y_i(0) + t_i y_i(1)$. Additionally, let $\mathbf{x}_i \in \mathcal{R}^d$ denote the ``as is'' set of covariates for the $i$-th individual. When the context is clear, we omit the subscript $i$ in the notations.

Throughout the paper, we assume that the following three fundamental assumptions for treatment effect estimations \cite{Rosenbaum1983} are satisfied:
\newtheorem{assumption}{Assumption}
\begin{assumption}
	(\textbf{SUTVA)} The Stable Unit Treatment Value Assumption requires that the potential outcomes for one unit (individual) is unaffected by the treatment of others.
\end{assumption}

\begin{assumption}
	(\textbf{Unconfoundedness)} The distribution of treatment is independent of the potential outcome when conditioning on the observed variables: $t \bigCI (y(0), y(1)) | \mathbf{x}$.
	\label{unconfound}
\end{assumption}

\begin{assumption}
	(\textbf{Overlap)} Every unit has a non-zero probability to receive either treatment or control when given the observed variables, i.e., $0< P(t=1|\mathbf{x})<1$.
	\label{overslap}
\end{assumption}

The first goal of treatment effect estimation is estimating the average treatment effect (ATE) which is defined as: $ATE = \mathbb{E}[y(1) - y(0)] = \mathbb{E}[y|do(t=1)] - \mathbb{E}[y|do(t=0)]$,
where $do(t=1)$ denote an manipulation on $t$ by removing all its incoming edges and setting $t=1$ \cite{Pearl2009}.

The treatment effect for an individual $i$ is defined as $\tau_i = y_i(1) - y_i(0)$.
Due to the counterfactual problem, we never observe $y_i(1)$ and $y_i(0)$ simultaneously and thus $\tau_i$ is not observed for any individual. Instead, we estimate the conditional average treatment effect $(\tau(x))$, defined as $\tau(\mathbf{x}) \vcentcolon = \mathbb{E}[\tau|\mathbf{x}] =  \mathbb{E}[y|\mathbf{x}, do(t=1)] - \mathbb{E}[y|\mathbf{x},do(t=0)]$.


\subsection*{Treatment Effect Estimation from Latent Factors}
%
%
%

In this work, we propose the TEDVAE model (Figure \ref{illustration}) for estimating the treatment effects, where the observed pre-treatment variables $\mathbf{x}$ can be viewed as generated from three disjoint sets of latent factors $\mathbf{z} = (\mathbf{z}_t,\mathbf{z}_c, \mathbf{z}_y)$. 
Here $\mathbf{z}_t$ are instrumental factors that only affect the treatment but not the outcome, 
$\mathbf{z}_y$ are risk factors which only affect the outcome but not the treatment, 
and $\mathbf{z}_c$ are confounding factors that affect both the treatment and the outcome. 

On the one hand, the proposed TEDVAE model in Figure \ref{illustration} provides two important benefits. The first one is that by explicitly modelling for the instrumental factors and adjustment factors, it accounts for the fact that not all variables in the observed variables set $\mathbf{x}$ are confounders. The second benefit is that it allows for the possibility of learning unobserved confounders that from their proxy variables.

On the other hand, our model diagram does not pose any restriction other than the three standard assumptions discussed in Section 3.1. 
To see this, consider the case where every variable in $\mathbf{x}$ itself is a confounder, i.e., $\mathbf{x}=\mathbf{x}_c$, then the generating mechanism in the TEDVAE model becomes $\mathbf{z_c} = \mathbf{x}$ with $\mathbf{z}_t=\mathbf{z}_y=\emptyset$ and the model in Figure \ref{illustration} becomes identical to the widely used diagram for treatment effect estimation (Figure 2 in \cite{Imbens2019}).

With our model, the estimation of treatment effect is immediate using the back-door criterion \cite{Pearl2009}:
\begin{theorem}
	The effect of $t$ on $y$ can be identified if we recover the confounding factors $\mathbf{z}_c$ from the data.
	\label{ATE}
\end{theorem}
\begin{proof}
	From Figure \ref{illustration} we know that $\mathbf{z}_t, \mathbf{z}_c$ are the parents of the treatment $t$, following (3.13) in \citeauthor{Pearl2009} we have,
	\begin{equation}
	P(y|do(t)) = \sum_{\mathbf{z}_t} \sum_{\mathbf{z}_c} P(y|t, \mathbf{z}_t,\mathbf{z}_c) P(\mathbf{z}_t)P(\mathbf{z}_c).
	\end{equation}
	Utilizing the fact that $y \bigCI \mathbf{z_t} |t, \mathbf{z_c}$, we have
	\begin{equation}
	\resizebox{0.89\hsize}{!}{
	$P(y|do(t)) = \sum\limits_{\mathbf{z}_t} P(\mathbf{z}_t) \sum\limits_{\mathbf{z}_c} P(y|t, \mathbf{z}_c) P(\mathbf{z}_c | t, \mathbf{z}_c, \mathbf{z}_t).$}
	\end{equation}
	Furthermore, since $\mathbf{z}_c$ is not a descendant of $t$, by Markov property we have $t \bigCI \mathbf{z}_c | \mathbf{z}_c,\mathbf{z}_t$. Therefore
	\begin{equation}
	P(y|do(t)) = \sum\limits_{\mathbf{z}_t} P(\mathbf{z}_t) \sum\limits_{\mathbf{z}_c} P(y|t, \mathbf{z}_c) P(\mathbf{z}_c | \mathbf{z}_c, \mathbf{z}_t).
	\end{equation}
	Note that $\sum\limits_{\mathbf{z}_t} P(\mathbf{z}_t) P(\mathbf{z}_c|\mathbf{z}_t,\mathbf{z_c}) = P(\mathbf{z}_c)$, which gives us
	\begin{equation*}
	P(y|do(t)) =\sum\limits_{\mathbf{z}_c} P(y|t,\mathbf{z}_c) P(\mathbf{z}_c).  \qedhere
	\end{equation*}
\end{proof}

For the estimation of the conditional average treatment effect, our result follows from Theorem 3.4.1 in \cite{Pearl2009} as shown in the following theorem:
\begin{theorem}
	The conditional average treatment effect of $t$ on $y$ conditioned on $\mathbf{x}$ can be identified if we recover the confounding factors $\mathbf{z}_c$ and risk factors $\mathbf{z}_y$ .
	\label{ITE}
\end{theorem}
\begin{proof}
	Let $G_{\overline{t}}$ denote the causal structure obtained by removing all incoming edges of $t$ in Figure \ref{illustration}, $G_{\underline{t}}$ denote the structure by deleting all outgoing edges of $t$. 
	
	\noindent Noting that $y \bigCI \mathbf{z}_t | t, \mathbf{z}_y, \mathbf{z}_c$ in $G_{\overline{t}}$, using the three rules of do-calculus we can remove $\mathbf{z}_t$ from the conditioning set and obtain 	$P(y|do(t),\mathbf{x}) = P(y|do(t), \mathbf{z}_t, \mathbf{z}_c, \mathbf{z}_y) = P(y | do(t), \mathbf{z}_y, \mathbf{z}_c)$.
	with Rule 1. Furthermore, using Rule 2 with $(y \bigCI t | \mathbf{z}_c, \mathbf{z}_y)$ in $G_{\underline{t}}$ yields $P(y|do(t),\mathbf{x}) =P(y | do(t), \mathbf{z}_y, \mathbf{z}_c) = P(y | t, \mathbf{z}_y, \mathbf{z}_c)$. \qedhere
\end{proof}

\begin{figure*}[!t]
	\centering
	\begin{subfigure}[b]{0.96\columnwidth}
		\centering
		\includegraphics[height=1.8in]{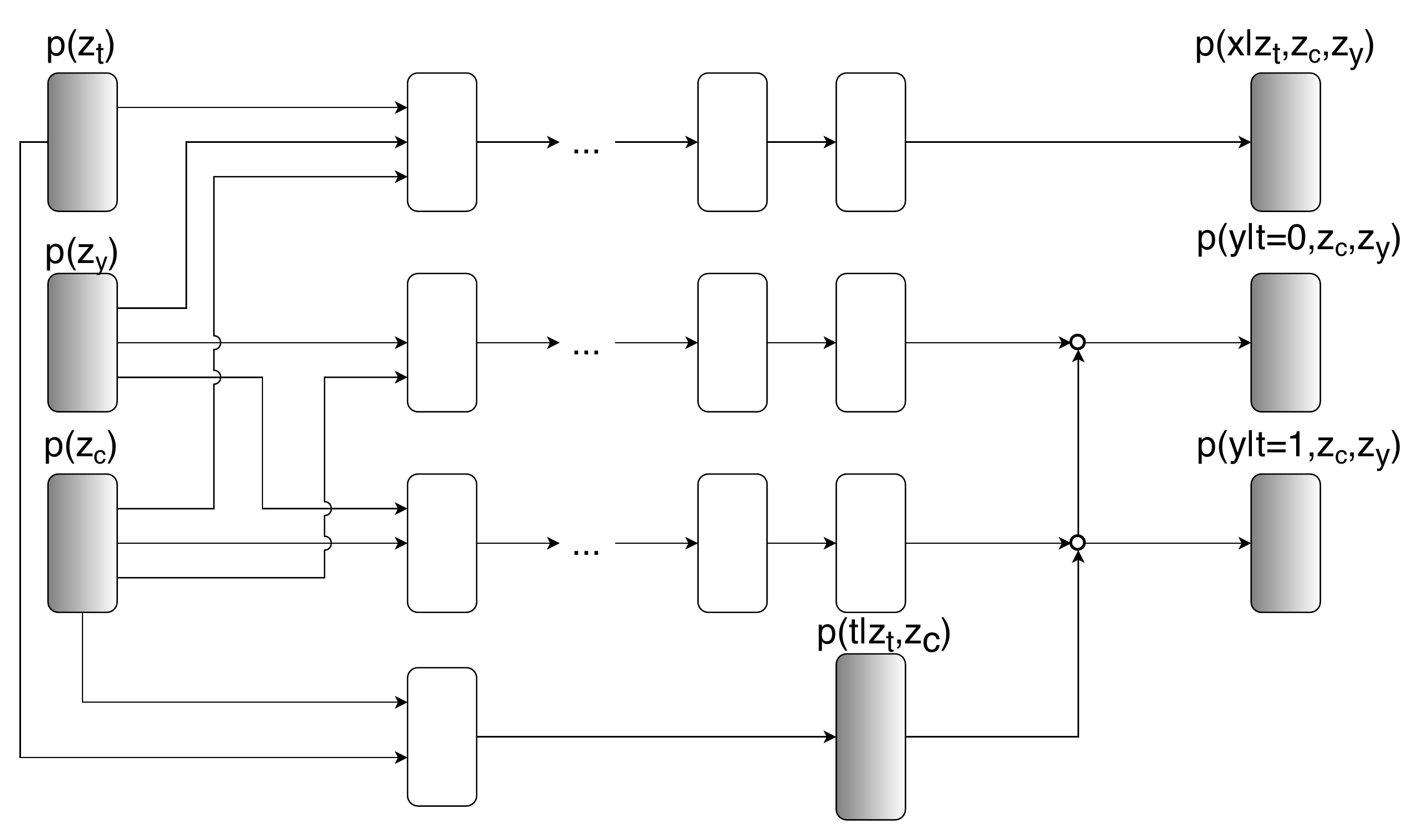}
		\caption{Generative Model.}
	\end{subfigure}
	~
	\begin{subfigure}[b]{0.96\columnwidth}
		\centering
		\includegraphics[height=1.6in]{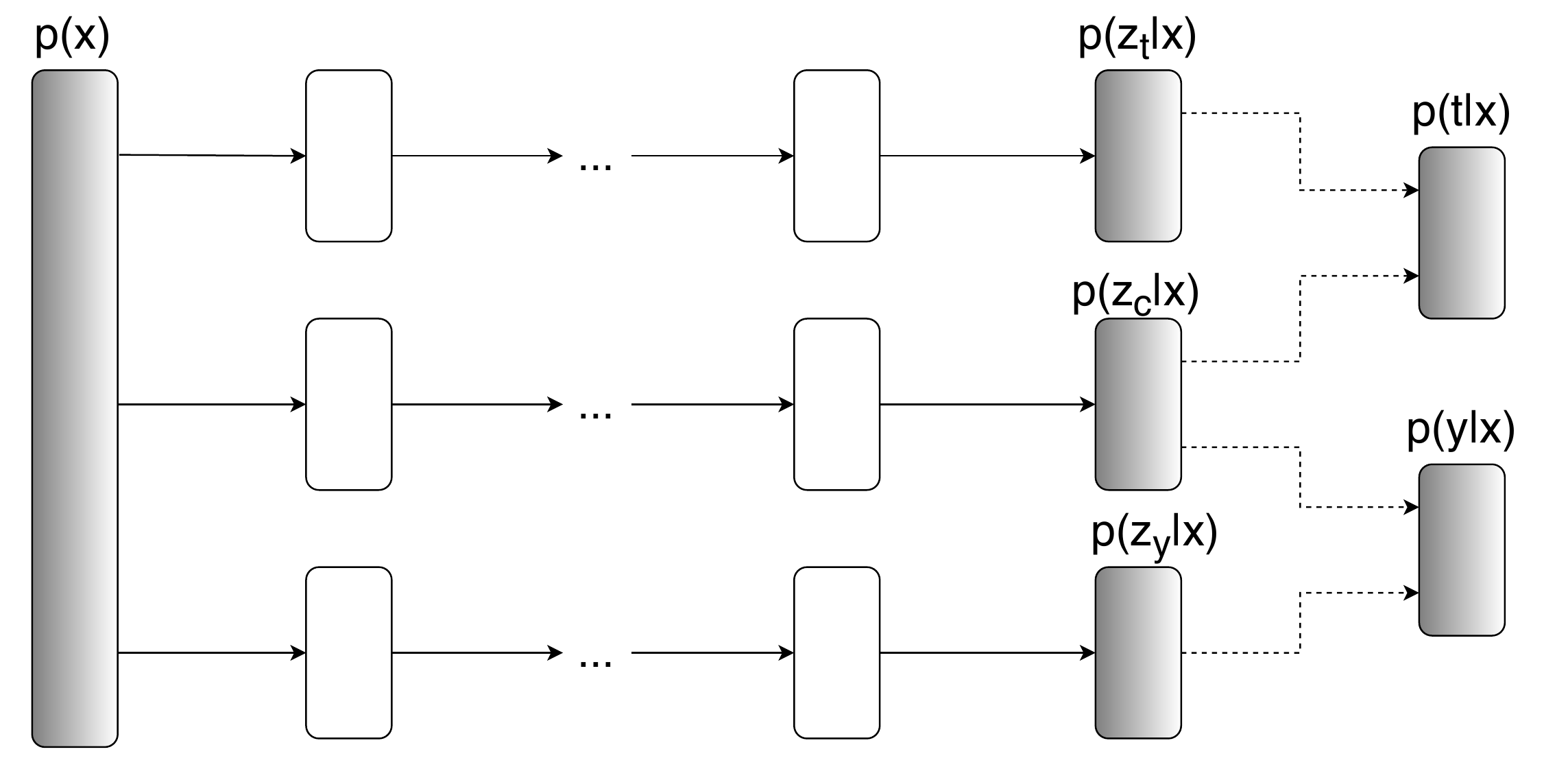}
		\caption{Inference Model.}
	\end{subfigure}
	\caption{Overall architecture of the model network and the inference network for the Treatment  Effect Disentangling Variational AutoEncoder (TEDVAE). White nodes correspond to parametrized deterministic neural network transitions, gray nodes correspond to drawing samples from the respective distribution and white circles correspond to switching paths according to the treatment $t$. Dashed arrows in the inference model represent the two auxiliary classifiers $q_{\omega_t}(t|\mathbf{z}_t,\mathbf{z}_c)$ and $q_{\omega_y} (y|\mathbf{z}_y,\mathbf{z}_c)$.}
	\label{architecture}
\end{figure*}

An implication of Theorem \ref{ATE} and \ref{ITE} is that they are not restricted to binary treatment. In other words, our proposed method can be used for estimating treatment effect of a continuous treatment variable, while most of the existing estimators are not able to do so. However, due to the lack of datasets with continuous treatment variables for evaluating this, we focus on the case of binary treatment variable and leave the continuous treatment case for future work.

Theorems \ref{ATE} and \ref{ITE} suggest that disentangling the confounding factors allows us to exclude unnecessary factors when estimating ATE and CATE. However, keen readers may wonder since we already assumed unconfoundedness, doesn't straightforwardly adjusting for $\mathbf{x}$ suffice? 

Theoretically, it has been shown that both the bias \cite{Abadie2006} and the variance \cite{Hahn1998} of treatment effect estimation will increase if variables unrelated to the outcome is included during the estimation. Therefore, it is crucial to differentiate the instrumental, confounding and risk factors and only use the appropriate factors during treatment effect estimation. 
In the next section, we propose our data-driven approach to learn and disentangle the latent factors using a variational autoencoder.


\subsection*{Learning of the Disentangled Latent Factors}

In the above discussion, we have seen that removing unnecessary factors is crucial for efficient and accurate treatment effect estimation. We have assumed that the mechanism which generates the observed variables $\mathbf{x}$ from the latent factors $\mathbf{z}$ and the decomposition of latent factors $\mathbf{z} = (\mathbf{z}_t, \mathbf{z}_c, \mathbf{z}_y)$ are available.
However, in practice both the mechanism and the decomposition are not known. Therefore, the practical approach would be to utilize the complete set of available variables during the modelling to ensure the satisfaction of the unconfoundedness assumption, and utilize a data-driven approach to simultaneously learn and disentangle the latent factors into disjoint subsets.

To this end, our goal is to learn the posterior distribution $p(\mathbf{z}|\mathbf{x})$ for the set of latent factors  with $\mathbf{z}= (\mathbf{z}_t, \mathbf{z}_y, \mathbf{z}_c)$ as illustrated in Figure \ref{illustration}, where $\mathbf{z}_t,\mathbf{z}_c, \mathbf{z}_y$ are independent of each other and correspond the instrumental factors, confounding factors, and risk factors, respectively.
Because exact inference would be intractable, we employ neural network based variational inference to approximate the posterior $p_\theta(\mathbf{x}|\mathbf{z}_t, \mathbf{z}_c, \mathbf{z}_y)$. Specifically, we utilize three separate encoders $q_{\phi_t}(\mathbf{z}_t|\mathbf{x})$, $q_{\phi_c}(\mathbf{z}_c|\mathbf{x})$, and $q_{\phi_y}(\mathbf{z}_y|\mathbf{x})$ that serve as variational posteriors over the latent factors. 
These latent factors are then used by a single decoder $p_\theta(\mathbf{x}|\mathbf{z}_t,\mathbf{z}_c, \mathbf{z}_y)$ for the reconstruction of $\mathbf{x}$. Following standard VAE design, the prior distributions $p(\mathbf{z}_t),p(\mathbf{z}_c),p(\mathbf{z}_y)$ are chosen as Gaussian distributions \cite{Kingma2014}. 

Specifically, the factors and the generative models for $\mathbf{x}$ and $t$ are described as:
\begin{align}
&p(\mathbf{z}_t) = \prod\limits_{j=1}^{D_{z_t}} \mathcal{N}(z_{tj}|0,1);\quad
p(\mathbf{z}_c) = \prod\limits_{j=1}^{D_{z_c}} \mathcal{N}(z_{cj}|0,1); \nonumber\\
&p(\mathbf{z}_y) = \prod\limits_{j=1}^{D_{z_y}} \mathcal{N}(z_{yj}|0,1);\quad p(t|\mathbf{z}_t,\mathbf{z}_c) = Bern(\sigma(f_1(\mathbf{z}_t,\mathbf{z}_c))\nonumber\\
& p(\mathbf{x}|\mathbf{z}_t,\mathbf{z}_c, \mathbf{z}_y) = \prod\limits_{j=1}^{d} p(x_j|\mathbf{z}_t,\mathbf{z}_c, \mathbf{z}_y),
\end{align}
with $p(x_j|\mathbf{z}_t,\mathbf{z}_c, \mathbf{z}_y)$ being the suitable distribution for the $j$-th observed variable, $f_1$ is a function parametrized by neural network, and $\sigma(\cdot)$ being the logistic function, $D_{z_t},D_{z_c}$, and $D_{z_y}$ are the parameters that determine the dimensions of instrumental, confounding, and risk factors to infer from $\mathbf{x}$. 

For continuous outcome variable $y$, we parametrize it as using a Gaussian distribution with its mean and variance given by a pair of disjoint neural networks that defines $p(y|t=1,\mathbf{z}_c, \mathbf{z}_y)$ and $p(y|t=0,\mathbf{z}_c, \mathbf{z}_y)$. This pair of disjoint networks allows for highly imbalanced treatment. Specifically, for continuous $y$ we parametrize it as:
\begin{align}
&p(y|t, \mathbf{z}_c, \mathbf{z}_y)  = \mathcal{N}(\mu =\hat{\mu}, \sigma^2 = \hat{\sigma}^2),\nonumber\\
&\hat{\mu} = (tf_2(\mathbf{z}_c, \mathbf{z}_y) + (1-t)f_3(\mathbf{z}_c, \mathbf{z}_y)), \nonumber\\
&\hat{\sigma}^2 = (tf_4(\mathbf{z}_c, \mathbf{z}_y) + (1-t)f_5(\mathbf{z}_c, \mathbf{z}_y)),
\end{align}
where $f_2$, $f_3$, $f_4$, $f_5$ are neural networks parametrized by their own parameters.
The distribution for the binary outcome case can be similarly parametrized with a Bernoulli distribution. 

In the inference model, the variational approximations of the posteriors are defined as:
\begin{align}
&q_{\phi_{t}}(\mathbf{z_t}|\mathbf{x}) = \prod\limits_{j=1}^{D_{z_t}} \mathcal{N}(\mu = \hat{\mu}_{t}, \sigma^2 = \hat{\sigma}_t^2 );\nonumber\\
&q_{\phi_{c}}(\mathbf{z_c}|\mathbf{x}) = \prod\limits_{j=1}^{D_{z_c}} \mathcal{N} (\mu = \hat{\mu}_{c}, \sigma^2 = \hat{\sigma}_c^2 );\nonumber\\
&q_{\phi_{y}}(\mathbf{z_y}|\mathbf{x}) = \prod\limits_{j=1}^{D_{z_y}} \mathcal{N} (\mu = \hat{\mu}_{y}, \sigma^2 = \hat{\sigma}_y^2 )  
\end{align}
where $\hat{\mu}_t, \hat{\mu}_c, \hat{\mu}_y$ and $ \hat{\sigma}_t^2,  \hat{\sigma}_c^2,  \hat{\sigma}_y^2$ are the means and variances of the Gaussian distributions parametrized by neural networks similarly to the $\hat{\mu}$ and $\hat{\sigma}$ in the generative model. 

Given the training samples, the parameters can be optimized by maximizing the evidence lower bound (ELBO):

\begin{align}
\mathcal{L}_{\textrm{ELBO}}(\mathbf{x},y,t) = & \mathbb{E}_{q_{\phi_c}{q_{\phi_t}}{q_{\phi_y}}} [\log p_\theta (\mathbf{x}|\mathbf{z}_t, \mathbf{z}_c, \mathbf{z}_y)] \nonumber\\
& -  D_{KL} (q_{\phi_t}(\mathbf{z}_t|\mathbf{x})|| p_{\theta_t}(\mathbf{z}_t)) \nonumber\\
& -  D_{KL} (q_{\phi_c}(\mathbf{z}_c|\mathbf{x})|| p_{\theta_c}(\mathbf{z}_c)) \nonumber\\
& -  D_{KL} (q_{\phi_y}(\mathbf{z}_y|\mathbf{x})|| p_{\theta_y}(\mathbf{z}_y)).
\end{align}
To  encourage the disentanglement of the latent factors and ensure that the treatment $t$ can be predicted from $\mathbf{z}_t$ and $\mathbf{z}_c$, and the outcome $y$ can be predicted from $\mathbf{z}_y$ and $\mathbf{z}_c$,  we add two auxiliary classifiers to the variational lower bound. Finally, the objective of TEDVAE can be expressed as
\begin{align}
\mathcal{L}_{\text{TEDVAE}} = & \mathcal{L}_{\textrm{ELBO}}(\mathbf{x},y,t)\nonumber\\
& + \alpha_t \mathbb{E}_{q_{\phi_{t}} q\phi_{c}} [\log q_{\omega_t}(t|\mathbf{z}_t,\mathbf{z}_c)]\nonumber \\
&+  \alpha_y \mathbb{E}_{q_{\phi_{y}} q\phi_{c}} [\log q_{\omega_y}(y|t, \mathbf{z}_c, \mathbf{z}_y)],
\label{loss_function}
\end{align}
where $\alpha_t$ and $\alpha_y$ are the weights for the auxiliary objectives.

For predicting the CATEs of new subjects given their observed covariates $\mathbf{x}$, we use the encoders $q(\mathbf{z}_y|\mathbf{x})$ and $q(\mathbf{z}_c|\mathbf{x})$ to sample the posteriors of the confounding and risk factors for $l$ times, and average over the predicted outcome $y$ using the auxiliary classifier $q_{\omega_y}(y|t, \mathbf{z}_c, \mathbf{z}_y)$. 

An important difference between TEDVAE and CEVAE lies in their inference models. 
During inference, CEVAE depends on $t$, $x$ and $y$ for inferencing $\mathbf{z}$; in other words, CEVAE needs to estimate $p(t|\mathbf{x})$ and $p(y|t,\mathbf{x})$, inference $\mathbf{z}$ as $p(\mathbf{z}|t,y,\mathbf{x})$, and finally predict the CATE as $\hat{\tau} (\mathbf{x}) = \mathbb{E}[y|t=1,\mathbf{z}] - \mathbb{E}[t|y=0, \mathbf{z}]$.
The estimations of $p(t|\mathbf{x})$ and $p(y|t,\mathbf{x})$ are unnecessary since we assume that $t$ and $y$ are generated by the latent factors and inferencing the latents should only depend on $\mathbf{x}$ as in TEDVAE. 
As we later show in the experiments, this difference is crucial even when no instrumental or risk factors are present in the data. 



\section*{Experiments}
We empirically compare TEDVAE with traditional and neural network based treatment effect estimators. 
For traditional methods, we compare with tailor designed methods including Squared t-statistic Tree (t-stats) \cite{Su2009} and Causal Tree (CT) \cite{Athey2015}; ensemble methods including Causal Random Forest (CRF) \cite{Wager2018}, Bayesian Additive Regression Trees (BART) \cite{Hill2011}, and meta algorithm X-Learner \cite{Kuenzel2019} using Random Forest \cite{Breiman1984} as base learner (X-RF).  
For deep learning based methods, we compare with representation learning based methods including Counterfactual Regression Net (CFR) \cite{Shalit2016}, 
Similarity Preserved Individual Treatment Effect (SITE) \cite{Yao2018_Twin}, and with a deep learning variable decomposition method for Counterfactual Regression (DR-CFR) \cite{Hassanpour2020}.
We also compare with generative methods including Causal Effect Variational Autoencoder (CEVAE) \cite{Louizos2017} and GANITE \cite{Yoon2018}.
Parameters for the compared methods are tuned by cross-validated grid search on the value ranges recommended in the code repository. The code is available at https://github.com/WeijiaZhang24/TEDVAE.
\begin{figure}[!t]
	\centering
	\begin{subfigure}{0.15\textwidth}
		\includegraphics[width=\linewidth]{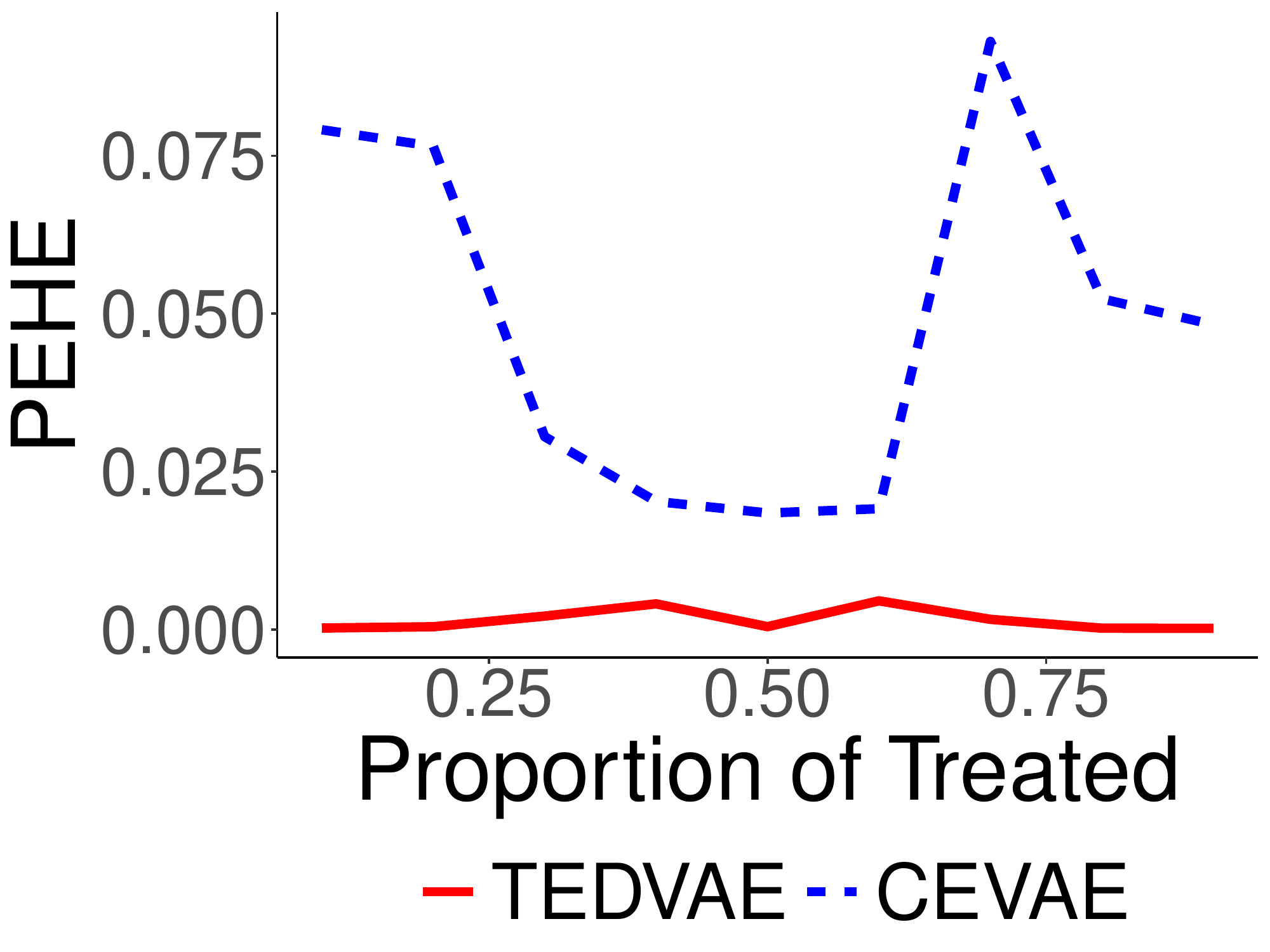}
	\end{subfigure}
	\begin{subfigure}{0.15\textwidth}
		\includegraphics[width=\linewidth]{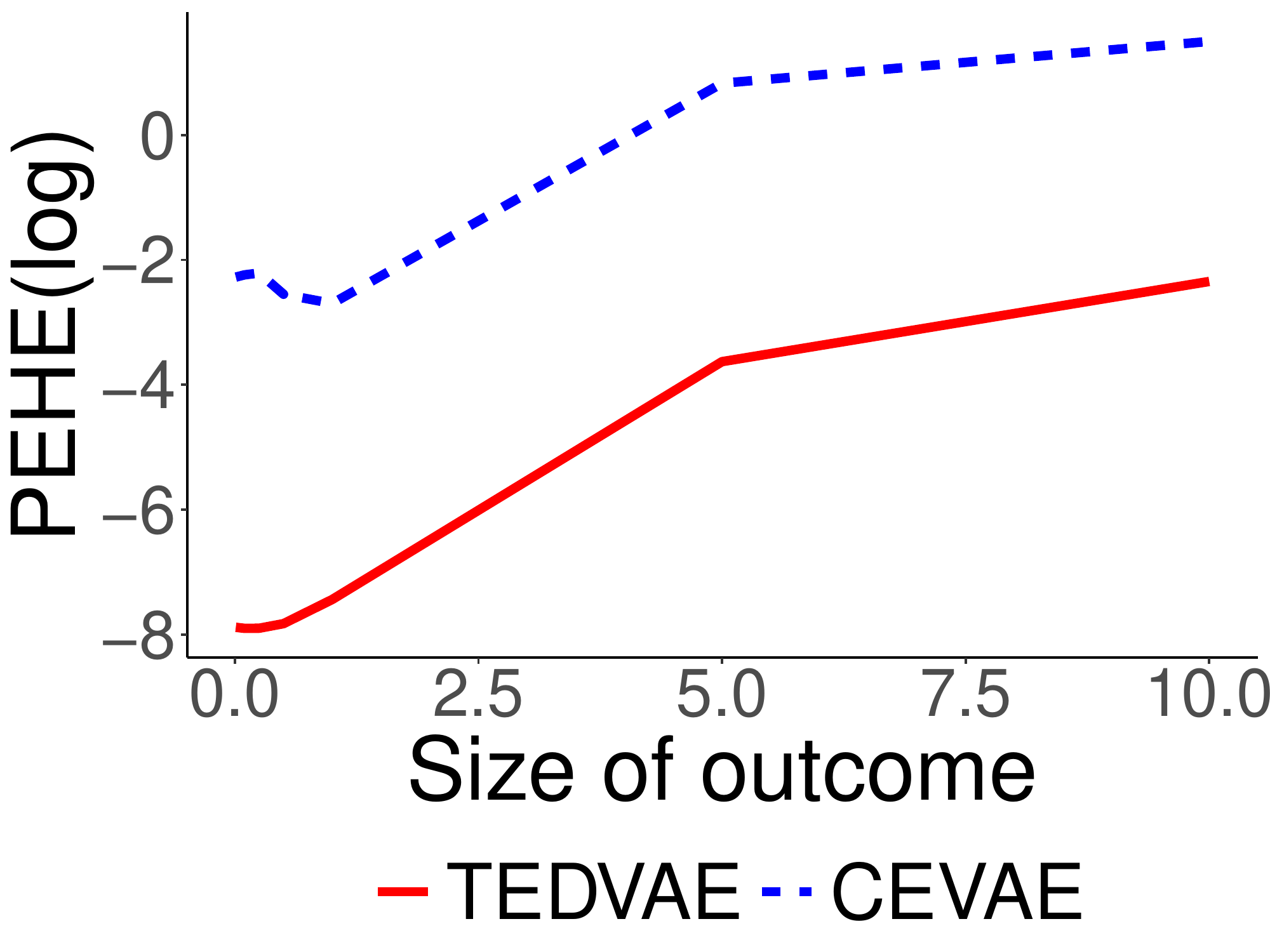}
	\end{subfigure}
	\begin{subfigure}{0.15\textwidth}
		\includegraphics[width=\linewidth]{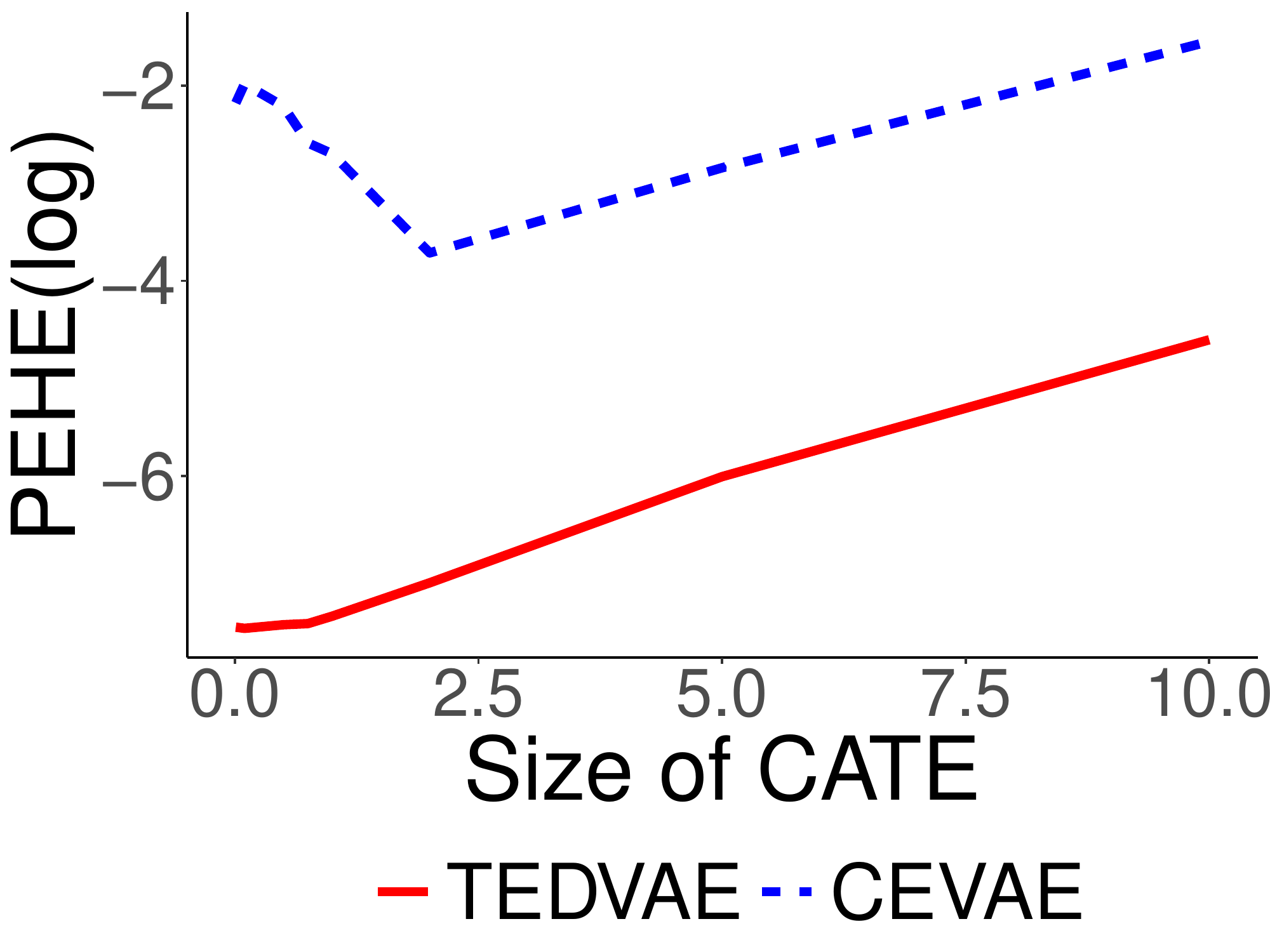}
	\end{subfigure}	
	\\
	\begin{subfigure}{0.15\textwidth}
		\includegraphics[width=\linewidth]{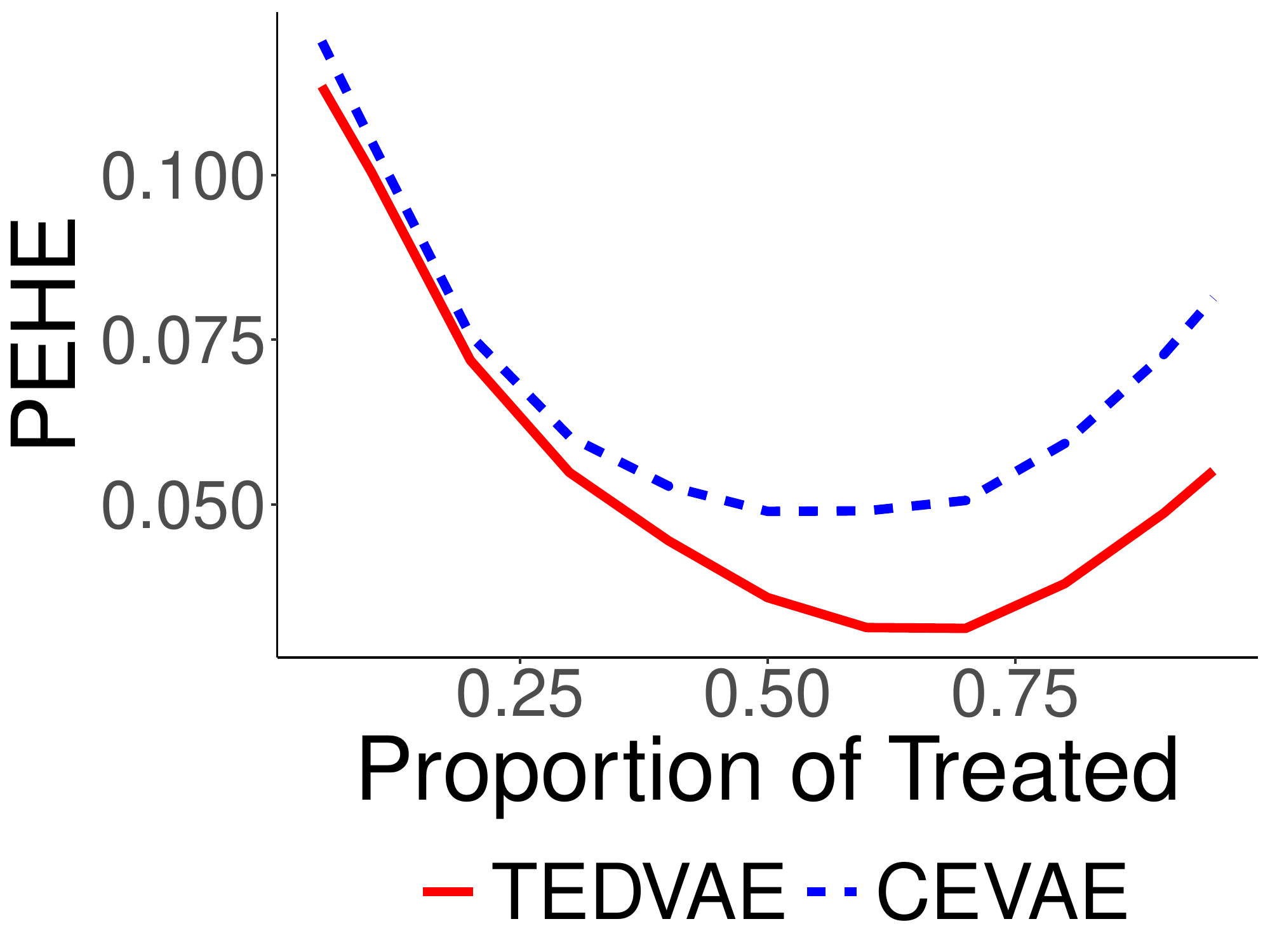}
	\end{subfigure}
	\begin{subfigure}{0.15\textwidth}
		\includegraphics[width=\linewidth]{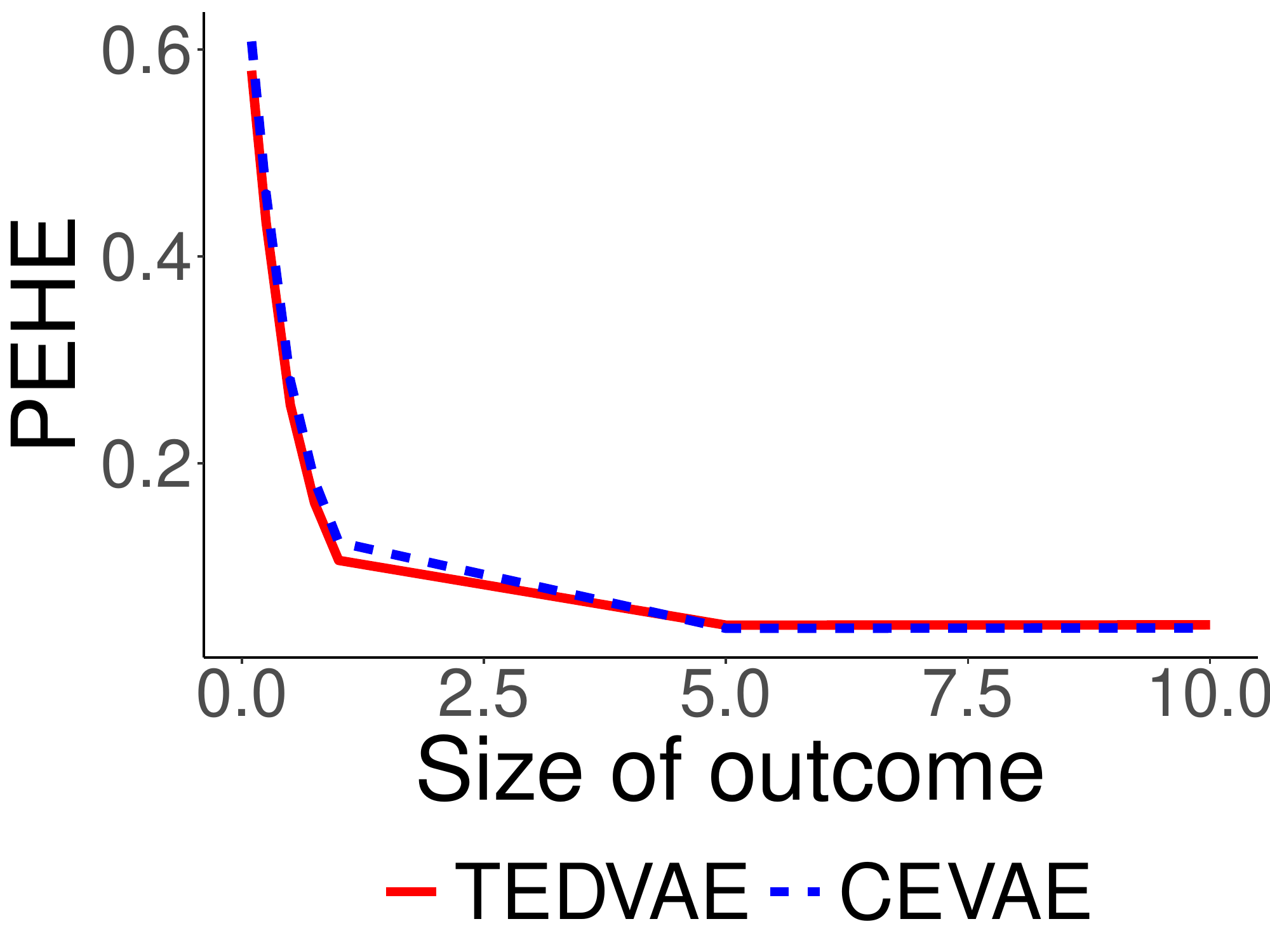}
	\end{subfigure}
	\begin{subfigure}{0.15\textwidth}
		\includegraphics[width=\linewidth]{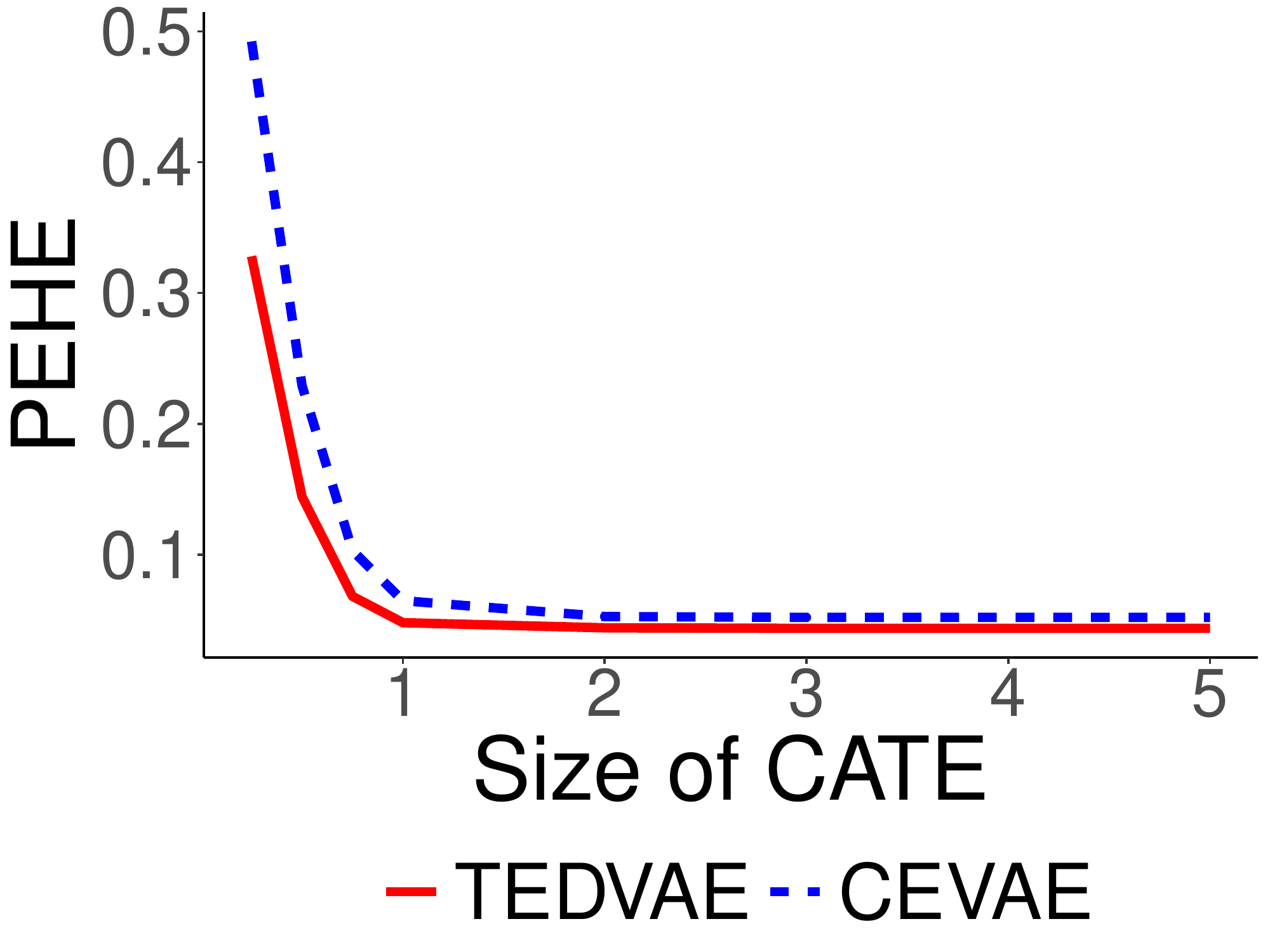}
	\end{subfigure}
	\caption{Comparison of CEVAE and TEDVAE under different settings using the synthetic datasets. Rows: the results for data generating procedure satisfies the assumption of the TEDVAE model and the CEVAE model, respectively. Columns: (Left) Varying the proportional of treated samples; (Middle) Varying the size of the outcome; (Right) Varying the size of the CATE. (Figures are best viewed in colour.)
	}
	\label{synthetic}
\end{figure}

\subsection*{Evaluation Criteria}

For evaluating the performance of CATE estimation, we use the Precision in Estimation of Heterogeneous Effect (PEHE) \cite{Hill2011,Shalit2016,Louizos2017,Dorie2019} which measures the root mean squared distance between the estimated and the true CATE when ground truth is available: $\epsilon_{\text{PEHE}} = \sqrt{\frac{1}{N}\sum_{i=1}^{N} (\hat{\tau}(\mathbf{x}_i) - \tau(\mathbf{x}_i))^2}$
%
, where $\tau(\mathbf{x})$ is the ground truth CATE for subjects with observed variables $\mathbf{x}_i$.

For evaluating the performance of the average treatment effect (ATE) estimation, the ground truth ATE can be calculated by averaging the differences of the outcomes in the treated and control groups if randomized controlled trials data is available. Then, when comparing the ground truth ATE with the estimated ATE obtained from a non-randomized sample (observational sample or created via biased sampling) of the dataset, the performances can then be evaluated using the mean absolute error in ATE \cite{Hill2011,Shalit2016,Louizos2017,Yao2018_Twin} for evaluation: $\epsilon_{\text{ATE}} = \vert \hat{\tau} - \frac{1}{N}\sum\limits_{i=1}^N[t_i y_i - (1-t_i)y_i]\vert$, 
where $\hat{\tau}$ is the estimated ATE, $t_i$ and $y_i$ are the treatments and outcomes from the randomized data. For both $\epsilon_{\text{ATE}}$ and $\epsilon_{\text{PEHE}}$, we use superscripts $tr$ and $te$ to denote their values on the training and test sets, respectively.

\subsection*{Synthetic Datasets}
We first conduct experiments using synthetic datasets to investigate TEDVAE's capability of inferring the latent factors and estimating the treatment effects. Due to the page limit, we only provide an outline of the synthetic dataset and provide the detailed settings in the supplementary materials.

The first setting of synthetic datasets studies the benefit of disentangling the confounding factors from instrumental and risk factors, and are generated using the structure depicted in Figure \ref{illustration}. 
We illustrate the results in the first row of Figure \ref{synthetic}. It can be seen that when the instrumental and risk factors exist in the data, the benefit of disentanglement is signficance as demonstrated by the PEHE curves between TEDVAE and CEVAE. When the proportions of the treated samples varies, the performances of CEVAE fluctuates severely and the error remains high even when the dataset is balanced; however, the PEHEs of TEDVAE are stable even when the dataset is highly imbalanced, and are always stays significantly lower than CEVAE. When the scales of outcome and CATE change, TEDVAE also performs consistently and significantly better than CEVAE.

The second setting for the synthetic datasets are designed to study how TEDVAE performs when the instrumental and risk factors are absent, and follow the same data generating procedure as in the CEVAE \cite{Louizos2017}.
We illustrate the results of this synthetic dataset in the second row of Figure \ref{synthetic}. Therefore, it is reasonable to expect that CEVAE would perform better than TEDVAE since the instrumental factors $\mathbf{z}_t$ and risk factors $\mathbf{z}_y$ do not exist.
However, from the second row of Figure \ref{synthetic} we can see that TEDVAE either performs better than  CEVAE, or performs as well as CEVAE using a wide range of parameters under this setting. 
This is possibly due to the differences in predicting for previous unseen samples between TEDVAE and CEVAE, where CEVAE needs to follow a complicated procedure of inferencing $p(t|\mathbf{x})$ and $p(y|t,\mathbf{x})$ first and then inferencing the latents as $p(z|t,y,\mathbf{x})$, whereas in TEDVAE this is not needed.
These results suggests that the TEDVAE model is able to effectively learn the latent factors and estimate the CATE even when the instrumental and risk factors are absent. It also indicates that the TEDVAE algorithm is robust to the selection of the latent dimensionality parameters.

Next, we investigate whether TEDVAE is capable of recovering the latent factors of $\mathbf{z}_t$, $\mathbf{z}_c$, and $\mathbf{z}_y$ that are used to generate the observed covariates $\mathbf{x}$. To do so, we compare the performances of TEDVAE when setting the $D_{z_t}$, $D_{z_c}$ and $D_{z_y}$ parameters to 10 against itself when setting one of the latent dimensionality parameter of TEDVAE to $0$, i.e., setting $D_{z_t}=0$ and forcing TEDVAE to ignore the existence of $\mathbf{z}_t$. If TEDVAE is indeed capable of recovering the latent factors, then its performances with non-zero latent dimensionality parameters should be better than its performance when ignoring the existence of any of the latent factors. 
Figure \ref{radar} illustrates the capability of TEDVAE for identifying the latent factors using radar chart. Taking the Figure \ref{radar}(a) as example, the $z_t$ and $\neg z_t$ polygons correspond to the performances of TEDVAE when setting the dimension parameter $D_{z_t}=5$ (identify $\mathbf{z}_t$) and $D_{z_t}=0$ (ignore $\mathbf{z}_t$). From the figures we can clearly see that the performances of TEDVAE are significantly better when the latent dimensionality parameters are set to non-zero, than setting any of the latent dimensionality to 0. 

\begin{figure}[!t]
	\centering
	\begin{subfigure}{0.15\textwidth}
		\includegraphics[width = \linewidth]{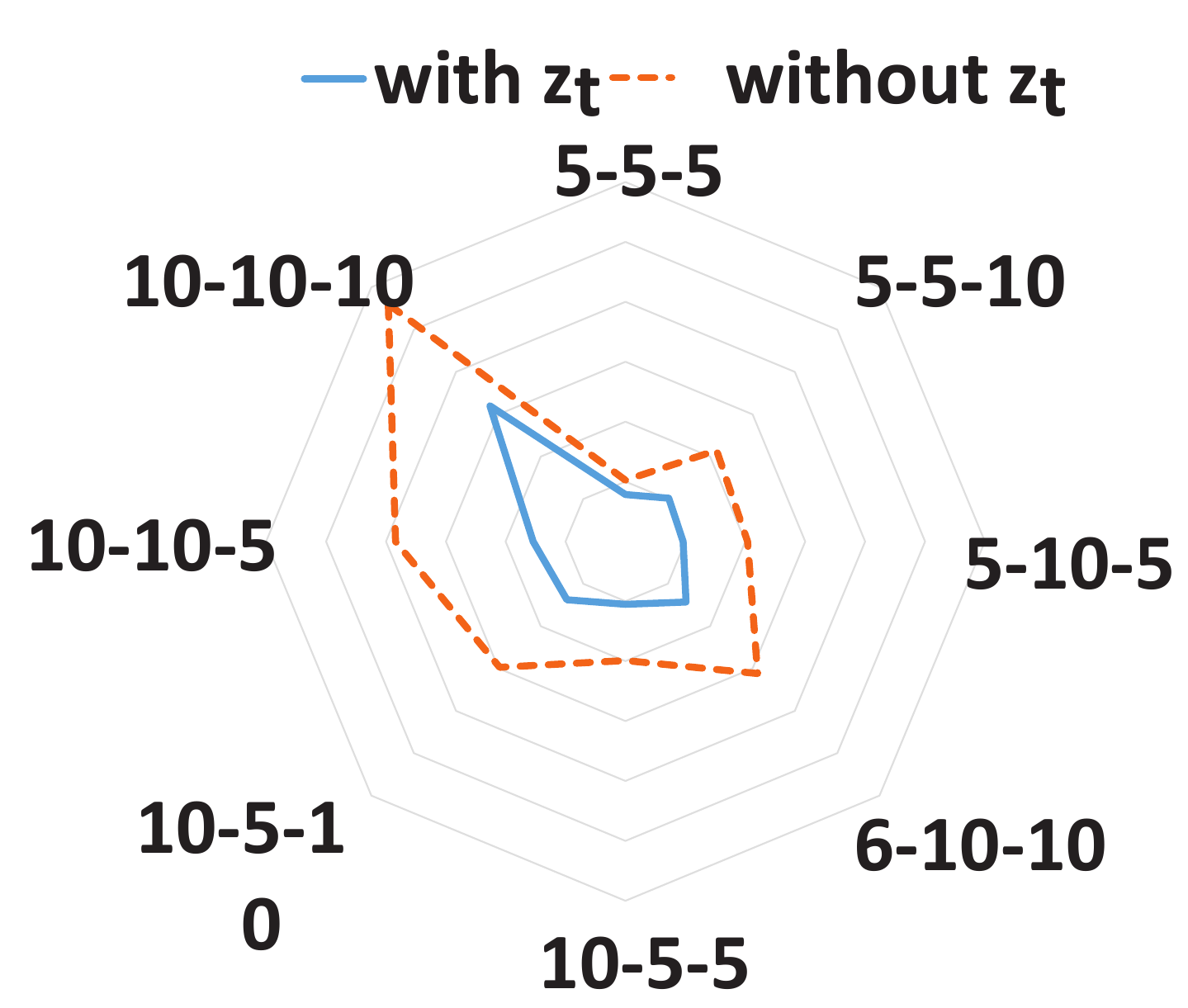}
		\caption{}
	\end{subfigure}
		\begin{subfigure}{0.15\textwidth}
			\includegraphics[width = \linewidth]{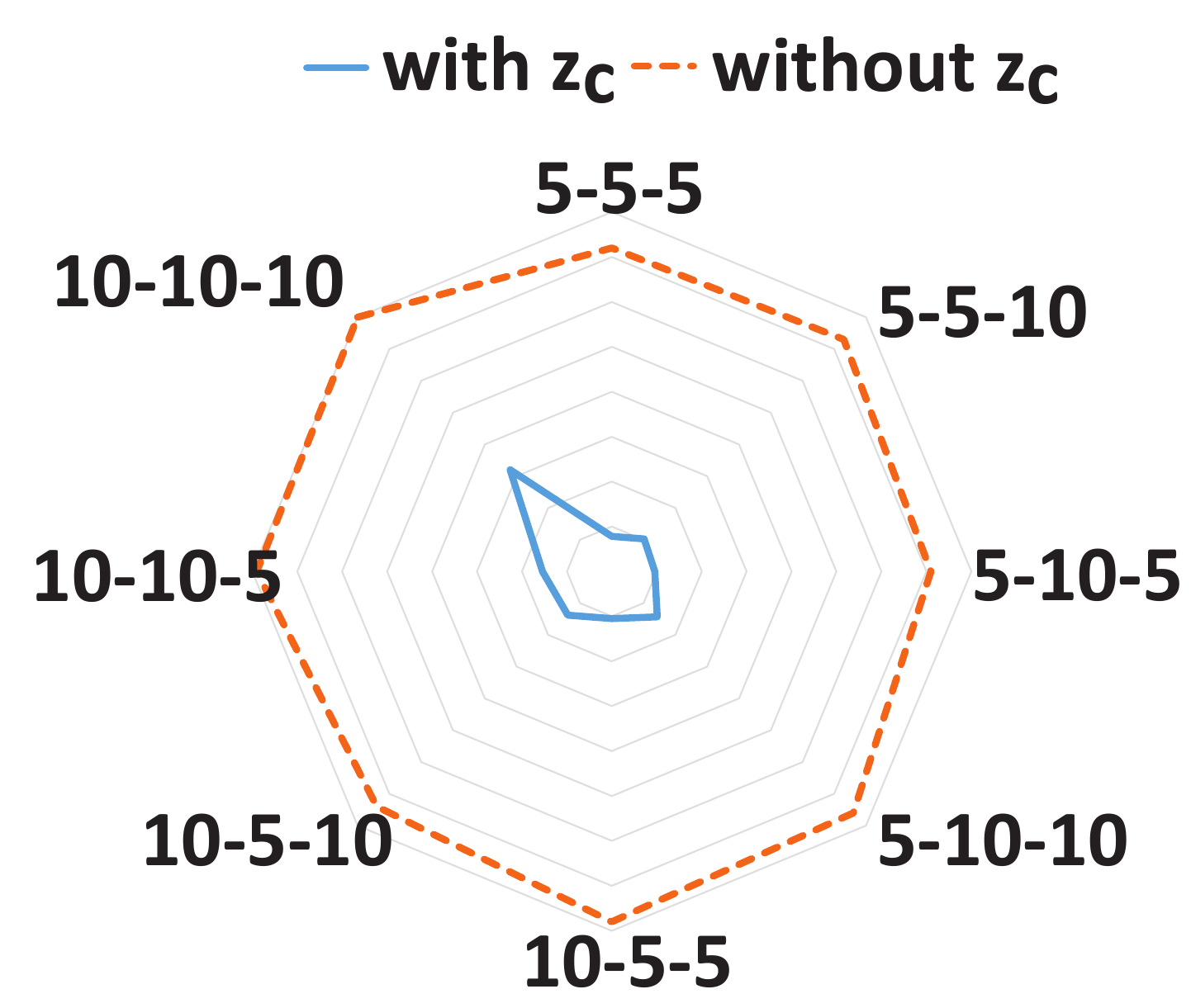}
			\caption{}
		\end{subfigure}
	\begin{subfigure}{0.15\textwidth}
		\includegraphics[width = \linewidth]{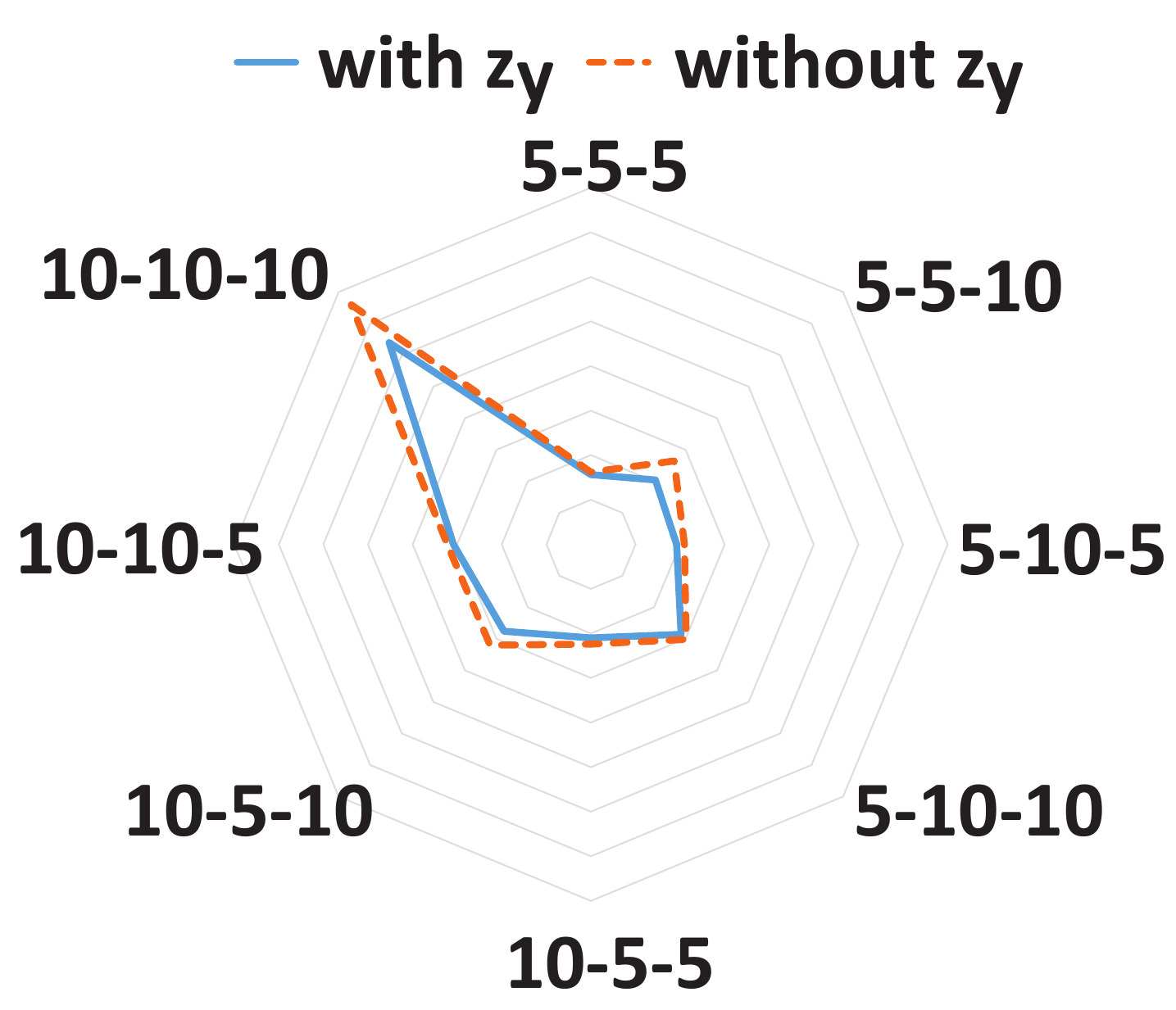}
		\caption{}
	\end{subfigure}
	\caption{Radar charts for TEDVAE's capability in identifying the latent factors. Each vertex on the polygons is identified with the latent factors’ dimension sequence of the associated synthetic dataset. For example, 5-5-5 indicates that the dataset is generated using 5 dimensions each for the instrumental, confounding, and risk factors.
	}
	\label{radar}
\end{figure}

\subsection*{Benchmarks and Real-world Datasets}
In this section, we use two benchmark datasets for treatment effect estimation to compare TEDVAE with the baselines. 
\subsubsection{Benchmark I: 2016 Atlantic Causal Inference Challenge}

\begin{table}[!t]
	\centering		
	\begin{tabular}{l | c c }
		\hline
		Methods & $\epsilon_{\text{PEHE}}^{tr}$ & $\epsilon_{\text{PEHE}}^{te}$ \\
		\hline
		CT &  4.81$\pm$0.18 & 4.96$\pm$0.21 \\
		t-stats &  5.18$\pm$0.18 & 5.44$\pm$0.20 \\
		\hline
		CF	& 2.16$\pm$0.17 & 2.18$\pm$0.19 \\
		BART & 2.13$\pm$0.18 & 2.17$\pm$0.15 \\
		X-RF & 1.86$\pm$0.15 & 1.89$\pm$0.16   \\
		\hline
		CFR	 & 2.05$\pm$0.18 & 2.18$\pm$0.20 \\
		SITE & 2.32$\pm$0.19 & 2.41$\pm$0.23  \\
		DR-CFR & 2.44$\pm$0.20& 2.56$\pm$0.21 \\
		\hline
		GANITE & 2.78$\pm$0.56 & 2.84$\pm$ 0.61 \\
		CEVAE & 3.12$\pm$0.28 & 3.28$\pm$0.35 \\
		TEDVAE & \textbf{1.75$\pm$0.14} &  \textbf{1.77$\pm$0.17}  \\
		\hline
	\end{tabular}
	\caption{Means and standard deviations of the PEHE metrics (smaller is better) for training and test sets on the 77 benchmark datasets from ACIC2016. The bolded values indicate the best performers (Wilcoxon signed rank tests at $p=0.05$).}
	\label{acic2016}
\end{table}

The 2016 Atlantic Causal Inference Challenge (ACIC2016) \cite{Dorie2019} contains 77 different settings of benchmark datasets that are designed to test causal inference algorithms under a diverse range of real-world scenarios. 
The dataset contains 4802 observations and 58 variables. The outcome and treatment variables are generated using different data generating procedures for the 77 settings, providing benchmarks for a wide range of treatment effect estimation scenarios. This dataset can be accessed at \url{https://github.com/vdorie/aciccomp/tree/master/2016}. 

We report the average PEHE metrics across 77 settings where each setting is repeated for 10 replications. For TEDVAE, the parameters are selected using the average of the first five settings, instead of tuning separately for the 77 settings. This approach has two benefits: firstly and most importantly, if an algorithm performs well using the same parameters across all 77 settings, it indicates that the algorithm is not sensitive to the choice of parameters and thus would be easier for practitioners to use in real-world scenarios; the second benefit is to save computation costs, as conducting parameter tuning across a large amount of datasets can be computationally overwhelming for practitioners.
As a result, we set the latent dimensionality parameters as $D_{z_y}=5$, $D_{z_t}=15$, $D_{z_c}=15$ and set the weight for auxiliary losses as $\alpha_t=\alpha_y=100$. For all the parametrized neural networks, we use 5 hidden layers and 100 hidden neurons in each layer, with ELU activation.   with a 60\%/30\%/10\% train/validation/test splitting proportions.

The results on the ACIC2016 datasets are reported in Table \ref{acic2016}. We can see that TEDVAE performs significantly better than the compared methods.
These results show that, without tuning parameters individually for each setting, TEDVAE achieves state-of-the-art performances across diverse range of data generating procedures, which empirically demonstrates that TEDVAE is effective for treatment effect estimation across different settings.

\begin{table}[!t]
	\centering		
	\setlength{\tabcolsep}{3pt}
	\begin{tabular}{l | c c | c c }
		\hline
		& \multicolumn{2}{c|}{Setting A} & \multicolumn{2}{c}{Setting B}\\
		\hline
		Methods & $\epsilon_{\text{PEHE}}^{tr}$ & $\epsilon_{\text{PEHE}}^{te}$ & $\epsilon_{\text{PEHE}}^{tr}$ & $\epsilon_{\text{PEHE}}^{te}$ \\
		\hline
		CT &  1.48$\pm$0.12 & 1.56$\pm$0.13 & 5.46$\pm$0.08 & 5.73$\pm$0.09\\				
		t-stats &  1.78$\pm$0.09 & 1.91$\pm$0.12 & 5.40$\pm$0.08 & 5.71$\pm$0.09\\
		\hline
		CF	& 1.01$\pm$0.08 & 1.09$\pm$0.16 & 3.86$\pm$0.05 & 3.91$\pm$0.07\\
		BART & 0.87$\pm$0.07 & 0.88$\pm$0.07 & 2.78$\pm$0.03 & 2.91$\pm$0.04\\
		X-RF & 0.98$\pm$0.08 & 1.09$\pm$0.15 & 3.50$\pm$0.04 & 3.59$\pm$0.06  \\
		\hline
		CFR	 & 0.67$\pm$0.02 & 0.73$\pm$0.04 & 2.60$\pm$0.04 & 2.76$\pm$0.04 \\
		
		SITE & 0.65$\pm$0.07 & 0.67$\pm$0.06 & 2.65$\pm$0.04 & 2.87$\pm$0.05 \\
		DR-CFR  & 0.62$\pm$0.15 & 0.65$\pm$0.18 & 2.73$\pm$0.04 & 2.93$\pm$0.05\\
		\hline	
		GANITE & 1.84$\pm$0.34 & 1.90$\pm$0.40 & 3.68$\pm$0.38 & 3.84$\pm$0.52\\
		CEVAE & 0.95$\pm$0.12 & 1.04$\pm$0.14 & 2.90$\pm$0.10 & 3.24$\pm$0.12\\
		TEDVAE & \textbf{0.59$\pm$0.11} & \textbf{0.60$\pm$0.14} & \textbf{2.10$\pm$0.09} & \textbf{2.22$\pm$0.08} \\
		\hline
	\end{tabular}
	\caption{Means and standard deviations of the PEHE metric (smaller is better) on IHDP. The bolded values indicate the best performers (Wilcoxon signed rank tests ($p=0.05$).}
	\label{IHDP}
\end{table}

\subsubsection{Benchmark II: Infant Health Development Program}
The Infant Health and Development Program (IHDP) is a randomized controlled study designed to evaluate the effect of home visit from specialist doctors on the cognitive test scores of premature infants. 
The datasets is first used for benchmarking treatment effect estimation algorithms in \cite{Hill2011}, where selection bias is induced by removing a non-random subset of the treated subjects to create an observational dataset, and the outcomes are simulated using the original covariates and treatments. 
It contains 747 subjects and 25 variables that describe both the characteristics of the infants and the characteristics of their mothers. 
We use the same procedure as described in \cite{Hill2011} which includes two settings of this benchmark: `Setting A" and ``Setting B", where the outcomes follow linear relationship with the variables in ``Setting A" and exponential relationship in ``Setting B". The datasets can be accessed at \url{https://github.com/vdorie/npci}.
The reported performances are averaged over 100 replications with a training/validation/test splits proportions of 60\%/30\%/10\%.

Since evaluating treatment effect estimation is difficult in real-world scenarios \cite{Alaa2019}, a good treatment effect estimation algorithm should perform well across different datasets with minimum requirement for parameter tuning.
Therefore, for TEDVAE we use the same parameters in the ACIC dataset and do not perform parameter tuning on the IHDP dataset. 
For the compared traditional methods, we also use the same parameters as selected on the ACIC benchmark. For the compared deep learning methods, we conduct grid search using the recommended parameter ranges from the relevant papers.

From Table \ref{IHDP} we can see that TEDVAE achieves the lowest PEHE errors among the compared methods on both Setting A and Setting B of the IHDP benchmark. Wilcoxon signed rank tests ($p=0.05$) indicate that TEDVAE is significantly better than the compared methods.
Since TEDVAE uses the same parameters on the IHDP datasets as in the previous ACIC benchmarks,  these results demonstrate that the TEDVAE model is suitable for diverse real-world scenarios and is robust to the choice of parameters.

\begin{table}[!t]
	\setlength{\tabcolsep}{4pt}
	\centering
	\begin{tabular}{l | c c}
		\hline
		&   \multicolumn{2}{c}{Twins} \\
		\hline
		Methods & $\epsilon_{\text{ATE}}^{tr}$ &  $\epsilon_{\text{ATE}}^{te}$  \\
		\hline
		CT &   0.034$\pm$0.002 & 0.038$\pm$0.007  \\
		t-stats & 0.032$\pm$0.003 & 0.033$\pm$0.005\\
		\hline
		CF	 & 0.025$\pm$0.001 & 0.025$\pm$0.001 \\
		BART  & 0.050$\pm$0.002 & 0.051$\pm$0.002\\
		X-RF   & 0.075$\pm$0.003 & 0.074$\pm$0.004 \\
		\hline
		CFR	    & 0.029$\pm$0.002 & 0.030$\pm$0.002\\
		
		SITE  & 0.031$\pm$0.003 & 0.033$\pm$0.003 \\ 
		DR-CFR & 0.032$\pm$0.002 & 0.034$\pm$0.003 \\
		\hline
		GANITE & 0.016$\pm$0.004 & 0.018$\pm$0.005 \\
		CEVAE   & 0.046$\pm$0.020 & 0.047$\pm$0.021 \\
		TEDVAE  & \textbf{0.006$\pm$0.002}  & \textbf{0.006$\pm$0.002} \\
		\hline
	\end{tabular}
	\label{ATE_results}
	\caption{ Means and standard deviations of $\epsilon_{\text{ATE}}$ on the Twins datasets. The bolded values indicate the best performers (Wilcoxon signed rank tests ($p=0.05$). }	
\end{table}

%
\subsubsection{Real-world Dataset: Twins}
In this section, we use a real-world randomized dataset to compare the methods capability of estimating the average treatment effects. 


The Twins dataset has been previously used for evaluating causal inference in \cite{Louizos2017,Yao2018_Twin}. It consists of samples from twin births in the U.S. between the year of 1989 and 1991 provided in \cite{Almond2005_TwinData}. Each subject is described by 40 variables related to the parents, the pregnancy and the birth statistics of the twins.
The treatment is considered as $t=1$ if a sample is the heavier one of the twins, and considered as $t=0$ if the sample is lighter. 
The outcome is a binary variable indicating the children's mortality after a one year follow-up period. 
Following the procedure in \cite{Yao2018_Twin}, we remove the subjects that are born with weight heavier than 2,000g and those with missing values, and introduced selection bias by removing a non-random subset of the subjects. The final dataset contains 4,813 samples. The data splitting is the same as previous experiments, and the reported results are averaged over 100 replications. 
The ATE estimation performances are illustrated in Table \ref{ATE_results}. 
On this dataset, we can see that TEDVAE achieves the best performance with the smallest $\epsilon_{ATE}$ among all the compared algorithms. 

Overall, the experiments results show that the performances of TEDVAE are significantly better than the compared methods on a wide range of synthetic, benchmark, and real-world datasets. In addition, the results also indicate that TEDVAE is less sensitive to the choice of parameters than the other deep learning based methods, which makes our method attractive for real-world application scenarios. 

\section*{Conclusion}
We propose the TEDVAE algorithm, a state-of-the-art treatment effect estimator which infer and disentangle three disjoints sets of instrumental, confounding and risk factors from the observed variables. 
Experiments on a wide range of synthetic, benchmark, and real-world datasets have shown that TEDVAE significantly outperforms compared baselines. 
For future work, a path worth exploring is extending TEDVAE for treatment effects with non-binary treatment variables. 
While most of the existing methods are restricted to binary treatments, the generative model of TEDVAE makes it a promising candidate for extension to treatment effect estimation with continuous treatments.

\bibliography{ijcai20}
\end{document}